\documentclass{article}

\usepackage{amsthm}
\usepackage{mathtools}
\usepackage[mathcal]{euscript}
\usepackage{amssymb}
\usepackage{bm}
\usepackage{subcaption}
\usepackage{booktabs}
\usepackage{array}
\usepackage{multirow}
\usepackage{float}
\usepackage{makecell}
\usepackage[ruled, vlined]{algorithm2e}
\usepackage{setspace}
\usepackage[margin=2cm]{geometry}
\usepackage{newtxtext}
\usepackage{newtxmath}
\usepackage[table]{xcolor}
\usepackage[absolute]{textpos}

\DeclarePairedDelimiter{\norm}{\lVert}{\rVert}

\makeatletter
\def\thmheadbrackets#1#2#3{%
	\thmname{#1}\thmnumber{\@ifnotempty{#1}{ }\@upn{#2}}%
	\thmnote{ {\the\thm@notefont[#3]}}}
\makeatother

\newtheoremstyle{brakets}% Name
{}% space above
{}% space below
{\itshape}% body font
{}% indent
{\bfseries}% head font
{.}% punctuation after head
{ }% space after head (has to be space or dimension!)
{\thmheadbrackets{#1}{#2}{#3}}% head spec

\newtheoremstyle{defbrakets}% Name
{}% space above
{}% space below
{\normalfont}% body font
{}% indent
{\bfseries}% head font
{.}% punctuation after head
{ }% space after head (has to be space or dimension!)
{\thmheadbrackets{#1}{#2}{#3}}% head spec

\newtheoremstyle{defproblem}% Name
{}% space above
{}% space below
{\normalfont}% body font
{}% indent
{\bfseries}% head font
{.}% punctuation after head
{ }% space after head (has to be space or dimension!)
{\thmheadbrackets{#1}{#2}{#3}}% head spec

\theoremstyle{brakets}

\newtheorem{thm}{Theorem}
\newtheorem*{thm-non}{Theorem} %theorem without number

\newtheorem{cor}[thm]{Corollary}
\theoremstyle{definition}
\theoremstyle{defbrakets}
\newtheorem{defn}{Definition}

\newtheorem*{plm-non}{Problem}%problem without number

\DeclareFontFamily{OT1}{pzc}{}
\DeclareFontShape{OT1}{pzc}{m}{it}{<-> s * [1.10] pzcmi7t}{}
\DeclareMathAlphabet{\mathpzc}{OT1}{pzc}{m}{it}

\makeatletter
\newcommand{\algrule}[1][.2pt]{\par\vskip.5\baselineskip\hrule height #1\par\vskip.5\baselineskip}
\makeatother

\title{\LARGE \bf
Meta Navigation Functions: Adaptive  Associations for Coordination of Multi-Agent Systems
}

\begin{document}

\title{Meta Navigation Functions: Adaptive  Associations for Coordination of Multi-Agent Systems}
\author{Matin~Macktoobian and Guillaume Ferdinand Duc}%
\date{The authors were with EPFL, Lausanne, Switzerland, at the time of this research. \\[2mm]Email:  matin.macktoobian@ualberta.ca}
\maketitle

\pagestyle{empty}

\begin{textblock}{14}(3,1)
	\noindent\textbf{\color{red}Published in ``American Control Conference 2022'' DOI: 10.23919/ACC53348.2022.9867570}
\end{textblock}

\begin{abstract}
In this paper, we introduce a new class of potential fields, i.e., meta navigation functions (MNFs) to coordinate multi-agent systems. Thanks to the MNF formulation, agents can contribute to each other's coordination via partial and/or total associations, contrary to traditional decentralized navigation functions (DNFs). In particular, agents may stimulate each other via their MNFs.  Moreover, MNFs need to be confined which is a weaker condition compared to the Morse condition of DNFs. An MNF is composed of a confined function and an attraction kernel. The critical points of the former can be confined in a safe region around a target critical point. The collision-free trajectory of an agent and its associations to its peers are governed by a confined function before reaching its safe region. Then, the attraction kernel drives the agent to its target in the safe region. MNFs provide faster coordination compared to DNFs. We illustrate how MNFs may exhibit some social behaviors in the course of partial and total associations among agents. Our simulations verify the efficiency of MNFs to coordinate complex swarms of agents.
\end{abstract}
\doublespacing
\section{Introduction}
The virtual-potential-based attraction/repulsion idea has been one of the most promising strategies for the coordination of multi-agent systems \cite{siciliano2016springer}. Accordingly, an agent and obstacles existing in its configuration space are attributed to a positive virtual potential, but the agent's target spot is marked by a negative one. Thus, the agent is eventually attracted by its target yet is repelled by the obstacles. Classic local minima problem \cite{warren1989global}, though, may rise in the case of complex multi-agent systems. Various methods have been proposed to bypass local minima issues such as harmonic potentials \cite{fahimi2009autonomous} and time-variant artificial potential fields \cite{macktoobian2016time}. However, a remarkable step to resolve the local minima issue was the emergence of navigation functions \cite{tanner2012multiagent}. 

A navigation function is basically an artificial potential field which meets four conditions the most important of which is Morse functionality \cite{palais2000generalized}. In short, a function is Morse if its critical points (CPs) are all non-degenerate. This assumption isolates CPs so that random perturbations easily destabilize those which are saddle and maximum points. Each agent may benefit from its own decentralized navigation function (DNF) \cite{de2006formation}. A DNF comprises a weighting factor whose clever setting drives all of the DNF's CPs, except the one representing the target point, toward obstacles. Thus, an obstacle's repulsive effect neutralizes the attraction of the CP located in its vicinity. Consequently, the agent is solely attracted by the desired CP.
\subsection{Motivation}
An agent coordinated by a DNF is in essence selfish because its attractive term only cares about its own convergence. This competitive manner of coordination is problematic in terms of completeness seeking in the coordination of various swarm robots such as astrobots \cite{macktoobian2021astrobotics,macktoobian2021data}. These swarms are used to generate the map of the observable universe to be used in the discovery of the dark energy nature. Namely, each astrobot is assigned to a particular target to receive its eminent rays from the space. So, the astrobot has to be coordinated while moving in a dense hexagonal formation of its peers. In this scenario, if an astrobot reaches its target, it permanently resides there. However, that astrobot may block its neighbors to reach their targets. Put differently, the convergence rate would generally not be maximized using (competitive) DNFs. In contrast, one may take cooperations among neighboring astrobots into account to tackle this issue. In particular, if an astrobot reaches its target yet blocking other peers, it may have to (at least temporarily) leave its target spot to give some space to its peers to cross that blocking point. Later, the astrobot may converge to its target spot again. Put differently, the coordination process is handled in a relatively unselfish framework. So, the idea of cooperation was proposed based on which the attractive term of the DNF associated with a particular agent also contributes to those of its neighboring peers \cite{macktoobian2019complete}. The generalization of such associations for mobile multi-agent systems is the topic of this paper.
\subsection{Methodology and Contributions}
This paper introduces the notion of meta navigation function (MNF) which includes a confined function and an attraction kernel only one of which governs the coordination of the corresponding agent at a time. A target spot is taken as a CP of an MNF instance. The coordination is done within two steps in two mutual regions. Namely, the safety region is the largest circular area with respect to the target CP. This region essentially comprises no obstacle. Moreover, the planning region is conceptually the configuration space of the problem from which the safety region is subtracted. The confined function first navigates the agent from its initial point to the CP which resides on the boundary of the safety region. To do so, one has to set the confinement factor of the confined function such that all of the CPs of the confined function reside in the safety region. In this step, all the potential collisions and associations between the agent and its peers are taken into account, and the attraction kernel is deactivated. Then, the agent is ready to enter the safety region in which there are no obstacles. The attraction kernel owns only one CP which is the target CP set for the confined function. Furthermore, the already-confined CPs of the confined function are invisible and neutral to the attraction kernel's functionality. Thus, once the attraction kernel is activated instead of the confined function, the agent moves toward its target CP in a direct line. Overall, the process described above features the following contributions to the coordination of multi-agent systems: \textit{Morse relaxation}, \textit{social behavioral implications}, and \textit{fast convergence}.
\section{Morse Relaxation through Criticality Confinement\protect\footnote{Throughout this paper, regular and bold symbols represent scalars and vectors, respectively.}}
\label{sec:mors-conf}
We focus on a general multi-agent system comprised of $N$ mobile
agents in $\mathbb{R}^{2}$ whose first-order dynamics read as
\begin{equation}
	\dot{q}_{i} = u_{i} \quad i = 1, \cdots, N,
\end{equation}
where $q_i$ and $u_i$ respectively denote the state and the control
input of agent $i$. Accordingly, a typical potential-based coordination method assigns each target spot to a CP.
\begin{defn}[Critical Set]
	Fix $W$ as a two-dimensional configuration space corresponding to the coordination problem of an agent. Let $\psi: \mathbb{R}^{2} \rightarrow \mathbb{R}^{\geqslant 0}$ be an arbitrary artificial potential field. Then, \textit{critical set} $\mathcal{C}$ is defined as below.
	\begin{equation}
	\mathcal{C} := \{\bm{q} \in W\mid\nabla \psi = \bm{0} \}
	\end{equation}
\end{defn}
\begin{defn}[DNF]
	Given configuration space $W$, let $\bm{q^{t}} \in W$ be the target point corresponding to a particular agent. Then, function $\psi: \mathbb{R}^{2} \rightarrow \mathbb{R}^{\geqslant 0}$ is a DNF associated with the agent if is simultaneously smooth, uniformly maximal on the boundary of its configuration space, minimum at $\bm{q} = \bm{q^{t}}\in \mathcal{C}$, and Morse.
\end{defn}
If the second derivative of the DNF associated with an agent vanishes with respect to a particular CP, the cited point is in fact a saddle point, thereby representing a local minima at which the agent may be trapped not to approach its desired minimum. Thus, the Morse condition is necessary since it guarantees the non-degeneracy of a DNF. In this section, we seek a weaker replacement for the Morse condition. For this purpose, we manage the local minima problem by the notion of criticality confinement. In particular, we plan to design a confined function the location of whose non-target CPs can be arbitrarily confined around a particular target CP. That confinement region enjoys the property of obstacle freeness, so once is reached, we deactivate the confined function, and an attraction kernel simply handles the final convergence phase of the coordination toward the target CP in the confinement region.
\begin{defn}[Boundary CP]
	Let $\mathcal{C}$ be the critical set associated with an artificial potential field. Then, \textit{boundary CP} $\bm{q^{\star}} \in \mathcal{C}$ corresponding to $\bm{q^{t}} \in \mathcal{C}$ is defined as follows.
	\begin{equation}
	\bm{q^{\star}} := \{\bm{q} \in \mathcal{C} \mid (\forall \bm{q'} \in \mathcal{C}) \norm{\bm{q^{\star}} - \bm{q^{t}}} > \norm{\bm{q} - \bm{q^{t}}} \}
	\end{equation}
\end{defn}
\begin{defn}[Critical Vector]
	Given a boundary CP $\bm{q^{\star}}$ corresponding to a target CP $\bm{q^{t}}$, \textit{critical vector} $\bm{v}$ with respect to $\bm{q^{t}}$ is defined as below.
	\begin{equation}
	\bm{v} := \bm{q^{t}} - \bm{q^{\star}}
	\end{equation}
\end{defn}
\begin{defn}[Critical Radius]
	Let $\bm{v}$ be a critical vector. Then, the critical radius is defined as the length of $\bm{v}$, say, $v = \norm{\bm{v}}$. 
\end{defn}
\begin{defn}[Critical Region]
	\label{defn:CritR}
	Let $\mathcal{C}$ be the critical set associated with an artificial potential field. Suppose $\bm{v}$ is the critical vector with respect to $\bm{q^{t}} \in \mathcal{C}$. Then, critical region $W_{|\bm{v}|} \subset W$ reads as follows.
	\begin{equation}
	W_{|\bm{v}|} := \{\bm{q} \in W \mid \norm{\bm{q^{t}}-\bm{q}}<\bm{|v|}\}
	\end{equation}
\end{defn}
\begin{defn}[Freeness Relation]
	Let $\mathcal{C}$ be the critical set associated with an artificial potential field. Given a $\bm{q^{t}} \in \mathcal{C}$, fix an $r \in \mathbb{R}^{\ge 0}$. If there is no object closer than distance $r$ to $\bm{q^{t}}$, then the binary relation $\mathcal{F}(\bm{q^{t}},r)$ holds.
\end{defn}
\begin{defn}[Confinement Radius]
	Let $\mathcal{C}$ be the critical set associated with an artificial potential field. Given a $\bm{q^{t}} \in \mathcal{C}$, the confinement radius $r$ with respect to $\bm{q^{t}}$ is defined as follows.
	\begin{equation}
	r := \max\{r' \mid (\forall r' \in \mathbb{R})~ \mathcal{F}(\bm{q^{t}},r')\}
	\end{equation}
\end{defn}
\begin{defn}[Confinement Region]
	\label{defn:ConfR}
	Let $\mathcal{C}$ be the critical set associated with an artificial potential field. Suppose $r$ is the confinement radius with respect to $\bm{q^{t}}$. Then, confinement region $W_{r} \subset W$ reads as follows.
	\begin{equation}
	W_{r} := \{\bm{q} \in W \mid \norm{\bm{q^{t}}-\bm{q}}<r\}
	\end{equation}
\end{defn}
Now, we define confined functions as below. 
\begin{defn}[Confined Function]
	Let $\mathcal{C}$ be the critical set associated with artificial potential field $\psi$. Given $\bm{q^{t}} \in \mathcal{C}$, let also $r$ be the confinement radius associated with $\bm{q^{t}}$. Then, if there exists a $0 < \delta \leqslant r$ such that
	$(\forall \bm{q} \in \mathcal{C}\setminus\{\bm{q^{t}}\})~\norm{\bm{q^{t}} - \bm{q}} < \delta
	$ holds, then $\psi$ is $(\bm{q^{t}},\delta)$-confined.
\end{defn}
We introduce a candidate confined function $\psi: \mathbb{R}^{2} \rightarrow \mathbb{R}^{\ge 0}$ as follows.
\begin{equation}
\label{eq:psi}
\psi(\bm{q_{i}};\alpha) :=  \overbrace{\rule{0pt}{6.8mm}\lambda_{1}\norm[\big]{\bm{q_{i}}-\bm{q^{t}}_{i}}^{2}}^{\text{attractive}} + \overbrace{\frac{\lambda_{2}}{\alpha}\sum_{\mathclap{\rule{0mm}{4mm} j\in[\mathcal{R}\setminus\{i\}]\dot{\cup}\mathcal{O}}}\dfrac{\norm[\big]{\bm{q_{i}}-\bm{q^{t}_{i}}}^{\frac{1}{\alpha}}}{\norm{\bm{q_{i}}-\bm{q_{j}}}^{2}\hfill}}^{\text{repulsive}}+\overbrace{\rule{0pt}{6.8mm}\lambda_{3}\norm[\big]{\bm{q_{i}}-\bm{q^{t}_{i}}}^{2}\displaystyle\sum_{\mathclap{k\in[\mathcal{R}\setminus\{i\}]}}\norm{\bm{q_{k}}-\bm{q^{t}_{k}}}^{2}}^{\text{associative}}
\end{equation}
This function includes three terms the first one of which represents the traditional attractive term of artificial potential fields. This term includes a real attraction factor $\lambda_{1} > 0$ to regulate the attraction force which is exclusively generated by the agent itself. The second term exhibits the repulsion functionality of the function in which $\mathcal{R}$ and $\mathcal{O}$ denote the set of all agents and the set of all obstacles associated with the coordination problem, respectively. The real repulsive factor $\lambda_{2} > 0$ may be used to tune the intensity of the repulsive force exerted to the agent. A real confinement factor $\alpha > 1$ is also taken into account in this term because, as we later see, a systematic selection of $\alpha$ gives rise to the desired confinement of the CPs of a candidate confined function with respect to a particular target CPt. We later observe that large confinement factors are those which can successfully confine CPs. The third term passes the associative impacts of other peers to the agent. In particular, the associative factor $\lambda_{3}$ regulates the amount of overall associative attraction applied to the agent. This expression transmits the attraction forces of other peers to the attractive dynamics of the agent. This consideration not only exhibits unselfishness of the agents but also generally makes the convergence of the agent faster to its target spot. On the other hand, one observes that the associative expression corresponding to each particular agent eventually vanishes when it reaches its target. This policy  first indicates that each agent's benevolence to other peers is not unlimited but restricted to the time it has not yet reached where it seeks. Second, this partial (un)selfish behavior does not postpone the overall convergence of multi-agent systems.

The theorem below presents a condition based on the confinement factor $\alpha$ whose fulfillment makes the candidate function (\ref{eq:psi}) a legit confined function. 
\begin{thm}
	Given an instance $\psi$ of the candidate confined function (\ref{eq:psi}), let $\mathcal{C}$ be the critical set corresponding to it. Let also $\bm{q^{t}} \in \mathcal{C}$ be the target CP associated with a coordination problem of a particular agent whose critical radius is $r$. Suppose $K$ and $J$ represent the index sets corresponding to the agent's peers and obstacles, respectively. Then, $\psi(\bm{q};\alpha)$ is $(\bm{q^{t}},r)$-confined for all $\alpha \geqslant \alpha^\dagger$ where $\alpha^\dagger$, known as the optimal confinement factor, is the solution to the following transcendental equation (called confinement condition)
	\begin{equation}
	\left[2(\lambda_{1}+ A\lambda_{3})\bm{v}\right]{\alpha^\dagger}^2 - [2\lambda_{2}\bm{C}]{\alpha^\dagger}r^{\frac{1}{\alpha^\dagger}} + [B\lambda_{2}r^{2}\bm{v}]r^{\frac{1}{\alpha^\dagger}} = \bm{0},
	\end{equation} 
	in which scalar parameters $A$ and $B$, and vector parameter $\bm{C}$ are defined, based on the boundary CP $q^{\star}$ associated with $\alpha = \alpha^\dagger$, as follows. 
	\begin{equation}
	A := \sum_{k \in K}\norm[\big]{\bm{q_{k}} - \bm{q_{k}^{t}}}^{2}, \quad B:= \sum_{j \in J}\frac{1}{\norm{\bm{q^{\star}} - \bm{q_{j}}}^3}, \bm{C} := \sum_{j \in J}\frac{\bm{q^{\star}} - \bm{q_{j}}}{\norm{\bm{q^{\star}} - \bm{q_{j}}}^4}
	\end{equation}
\end{thm}
\begin{proof}
	$\alpha$ has to be chosen such that all of the non-target CPs reside in the confinement region. As a lower bound of $\alpha$, suppose that taking $\alpha^\dagger$ into account yields the boundary CP residing at least on the boundary of the confinement region (and obviously remainder of the CPs inside it) which formally renders to $v = r$, meaning 
	\begin{equation}
	\label{eq:condition}
	\norm[\big]{\bm{q^{t}} - \bm{q^{\star}}} = r.
	\end{equation}
	Note that since $\bm{q^{\star}}$ is a CP, we must have $\nabla\psi(\bm{q^{\star}};\alpha^\dagger) = \bm{0}$. The expansion of the quoted equation gives
	\begin{equation}
	\label{eq:deriv}
	2(\bm{q^{\star}} - \bm{q^{t}})\lambda_{1} +
	\frac{\lambda_{2}}{\alpha^\dagger}\sum_{j \in J}\Big[\frac{\frac{1}{\alpha^\dagger}(\bm{q^{\star}} - \bm{q^{t}})\norm{\bm{q^{\star}} - \bm{q^{t}}}^{\frac{1}{\alpha^\dagger}} }{\norm{\bm{q^{\star}} - \bm{q_{j}}}^{4}}- \frac{2(\bm{q^{\star}} - \bm{q_{j}})\norm{\bm{q^{\star}} - \bm{q^{t}}}^{\frac{1}{\alpha^\dagger}} }{\norm{\bm{q^{\star}} - \bm{q_{j}}}^{4}}\Big]+
	 2(\bm{q^{\star}} - \bm{q^{t}})\lambda_{3}\sum_{k \in K}\norm[\big]{\bm{q_{k}} - \bm{q_{k}^{t}}}^{2} = \bm{0}.
	\end{equation}
%	\begin{equation}
%	\label{eq:deriv}
%	\begin{split}
%	& 2(\bm{q^{\star}} - \bm{q^{t}})\lambda_{1} + \\&
%	\frac{\lambda_{2}}{\alpha^\dagger}\sum_{j \in J}\Big[\frac{\frac{1}{\alpha^\dagger}(\bm{q^{\star}} - \bm{q^{t}})\norm{\bm{q^{\star}} - \bm{q^{t}}}^{\frac{1}{\alpha^\dagger}} }{\norm{\bm{q^{\star}} - \bm{q_{j}}}^{4}}-\\& \frac{2(\bm{q^{\star}} - \bm{q_{j}})\norm{\bm{q^{\star}} - \bm{q^{t}}}^{\frac{1}{\alpha^\dagger}} }{\norm{\bm{q^{\star}} - \bm{q_{j}}}^{4}}\Big]+\\
%	& 2(\bm{q^{\star}} - \bm{q^{t}})\lambda_{3}\sum_{k \in K}\norm[\big]{\bm{q_{k}} - \bm{q_{k}^{t}}}^{2} = \bm{0}.
%	\end{split}
%	\end{equation}
	Applying (\ref{eq:condition}) to the equation above yields
	\begin{equation}
	2\bm{v}\left[\lambda_{1}+ \lambda_{3}\sum_{k \in K}\norm[\big]{\bm{q_{k}} - \bm{q_{k}^{t}}}^{2}\right] + 
	\frac{\lambda_{2}}{{\alpha^\dagger}^2}\left[(\bm{q^{\star}} - \bm{q^{t}})\norm{\bm{q^{\star}} - \bm{q^{t}}}^{\frac{1}{\alpha^\dagger}-2}\right]\left[\sum_{j \in J}\frac{1}{\norm{\bm{q^{\star}} - \bm{q_{j}}}^3}\right]-\frac{2\lambda_{2}}{\alpha^\dagger}\left[\norm{\bm{q^{\star}} - \bm{q^{t}}}^{\frac{1}{\alpha^\dagger}}\right]\left[\sum_{j \in J}\frac{\bm{q^{\star}} - \bm{q_{j}}}{\norm{\bm{q^{\star}} - \bm{q_{j}}}^4}\right]= \bm{0}.
	\end{equation}
	We use the theorem's auxiliary parameters and multiply the resulting equation by ${\alpha^\dagger}^2$. So, we obtain
	\begin{equation}
	\left[2(\lambda_{1}+ A\lambda_{3})\bm{v}\right]{\alpha^\dagger}^2 - [2\lambda_{2}\bm{C}]{\alpha^\dagger}r^{\frac{1}{\alpha^\dagger}} + [B\lambda_{2}r^{2}\bm{v}]r^{\frac{1}{\alpha^\dagger}} = \bm{0},
	\end{equation}
	whose solution, i.e., $\alpha^\dagger$ is the lower bound of $\alpha$ according to which $\psi(\bm{q};\alpha)$ is $(\bm{q^{t}},r)$-confined in a guaranteed manner. 
	
	Now, it remains to show that any $\alpha > \alpha^\dagger$ holds the cited confinement, as well. Note that
	\begin{equation}
	\lim\limits_{\alpha \to \infty} \nabla\psi(\bm{q};\alpha) = (\bm{q^{\star}} - \bm{q^{t}}) \left[2\lambda_{1} +
	2\lambda_{3}\sum_{k \in K}\norm[\big]{\bm{q_{k}} - \bm{q_{k}^{t}}}^{2}\right].
	\end{equation}
	The limit gradient above only possesses one single CP $\bm{q} = \bm{q^{t}}$, that is, the critical radius is zero. Therefore, for any $\alpha^\dagger \leqslant \alpha<\infty$, $\psi(\bm{q};\alpha)$ is also $(\bm{q^{t}},r)$-confined, which is the claim of the theorem.
\end{proof}
Smoothness and admissibility are two requirements of a valid DNF both of which are the inherent properties of a confined function, as well. In particular, the control signal derived from any potential function used for the purpose of coordination is its weighted derivative. So, it has to be at least $\mathpzc{C}^2$. Moreover, a DNF has to be admissible on its domain. The configuration space of a coordination problem is often assumed to be circular and bounded \cite{de2006formation}. As well, the boundary maximality of DNFs comes from the fact that the boundary of their configuration space are often considered as obstacles. We similarly take the same assumption into account for MNFs. Thus, the uniform boundary maximality is automatically fulfilled for them. In particular, it only demands the non-degeneracy of its target CP.
\begin{cor}
	Given a critical radius $r$, let $\psi(\bm{q};\alpha)$ be a $(\bm{q^{t}},r)$-confined function corresponding to $\bm{q^{t}}$ target CP and confinement factor $\alpha$. Then, $\bm{q^{t}}$ is non-degenerate.
\end{cor}
\section{Meta Navigation Functions}
\label{sec:MNF}
We illustrated how a confined function confines its CPs in a confinement region around a particular target CPt. A confined function plans the path of its corresponding agent from an initial point to the boundary critical state located on (or inside) the confinement region. Thus, another potential entity, i.e., attraction kernel, is required to be activated instead of the confined function to finalize the coordination by reaching the desired target CP.
\begin{defn}[Attraction Kernel]
	\label{defn:kernel}
	Given a critical radius $r$, let $\psi(\bm{q};\alpha)$ be a $(\bm{q^{t}},r)$-confined function corresponding to the coordination of an agent. Let also $W_{r}$ be the confinement region associated with $\psi(\bm{q};\alpha)$. Then, \textit{attraction kernel} $\omega: W_{r} \rightarrow \mathbb{R}^{\geqslant 0}$ with respect to $\bm{q^{t}}$ is defined as
	\begin{equation}
	\omega(\bm{q}) := \beta\norm{\bm{q} - \bm{q^{t}}}.
	\end{equation} 
	Here, real $\beta > 0$ is a kernel factor.
\end{defn}  
\begin{defn}[Meta Navigation Function]
	Let $\psi(\bm{q};\alpha)$ be a $(\bm{q^{t}},\alpha)$-confined function corresponding to the coordination of an agent where $\bm{q^{t}}$ is a target CP associated with $\psi(\bm{q};\alpha)$. Let also $W_{r}$ be the confinement region associated with $\psi(\bm{q};\alpha)$. Denote by $\omega(\bm{q})$ an attraction kernel with respect to $\bm{q^{t}}$. Then, \textit{MNF} $\phi$ is defined as below.
	\begin{equation}
	\phi(\bm{q};\alpha) := 
	\begin{cases}
	\psi(\bm{q};\alpha) & \text{ if } \bm{q} \not\in W_{r}\\
	\omega(\bm{q}) & \text{otherwise}
	\end{cases}
	\end{equation}
	Let $\mathcal{P}$ be a predicate $\mathcal{P}: W_{r} \rightarrow \{0,1\}$ which can always be identified with the corresponding subset 
	\begin{equation}
	W_{r}^{\mathcal{P}} := \{\bm{q} \in W_{r} \mid  \mathcal{P}(\bm{q}) = 1 \}.
	\end{equation}
	Thus, MNF definition can be summarized as follows.
	\begin{equation}
	\phi(\bm{q};\alpha) := \psi(\bm{q};\alpha)\overline{\mathcal{P}(\bm{q})} + \omega(\bm{q})\mathcal{P}(\bm{q})
	\end{equation}
\end{defn}
\begin{cor}
	\label{prop:fast}
	The fastest coordination of agent $i$ corresponds to $\alpha = \alpha^\dagger$.
\end{cor}
Finally, the MNF-based Coordination of multi-agent systems may be achieved using the algorithm depicted in Fig. \ref{fig:coord}, in which GDC$()$ is a standard gradient descent computer.

\begin{figure*}[h]
	\centering
	\begin{subfigure}[b]{0.2\textwidth}
		\centering
		\hspace*{-2.7cm}\includegraphics[scale=0.33]{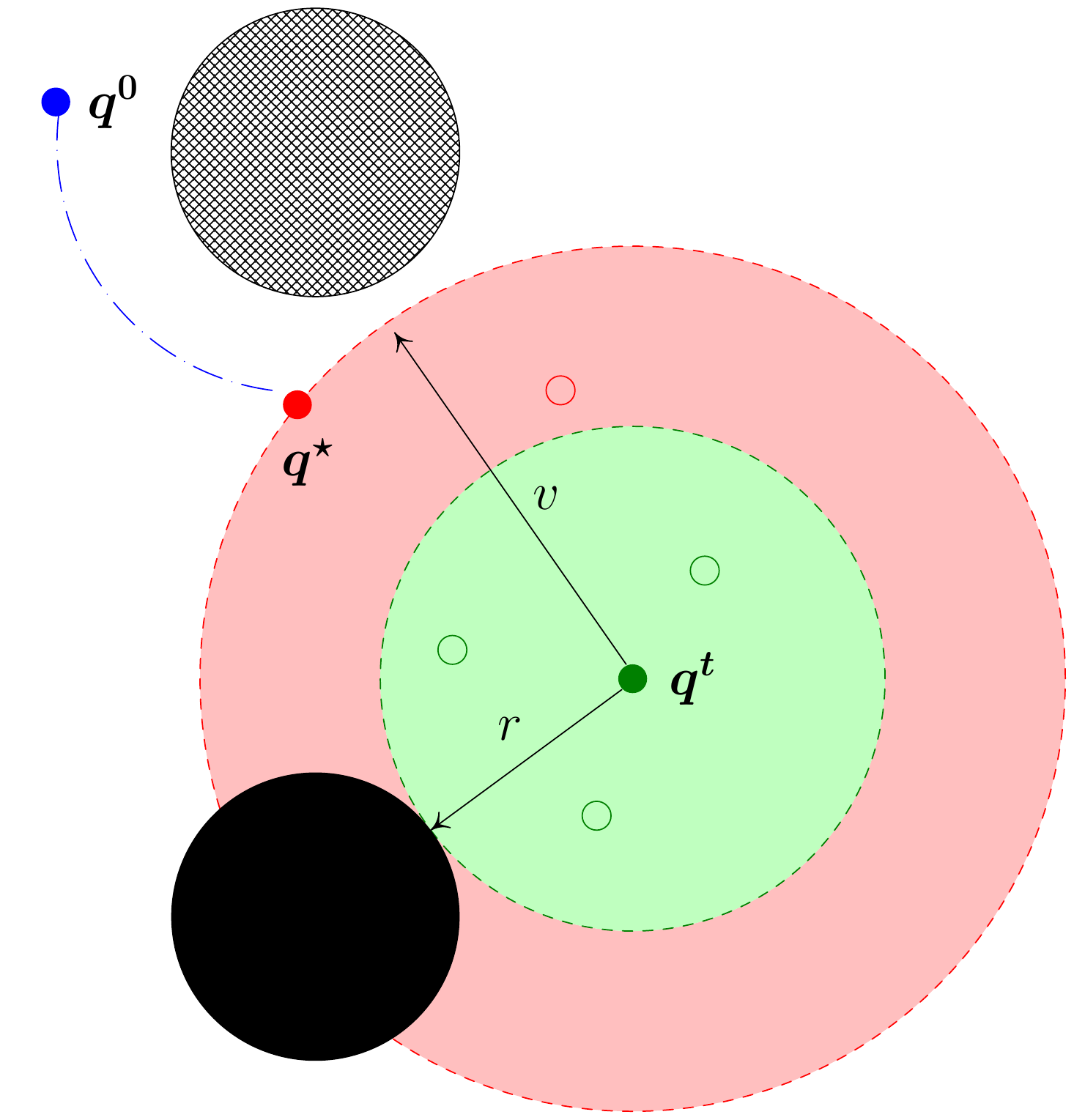}
		\caption[]%
		{{\small Before $\alpha$ determination}}    
		\label{fig:1}
	\end{subfigure}
	\begin{subfigure}[b]{0.2\textwidth}   
		\centering 
		\hspace*{-2cm}\includegraphics[scale=0.33]{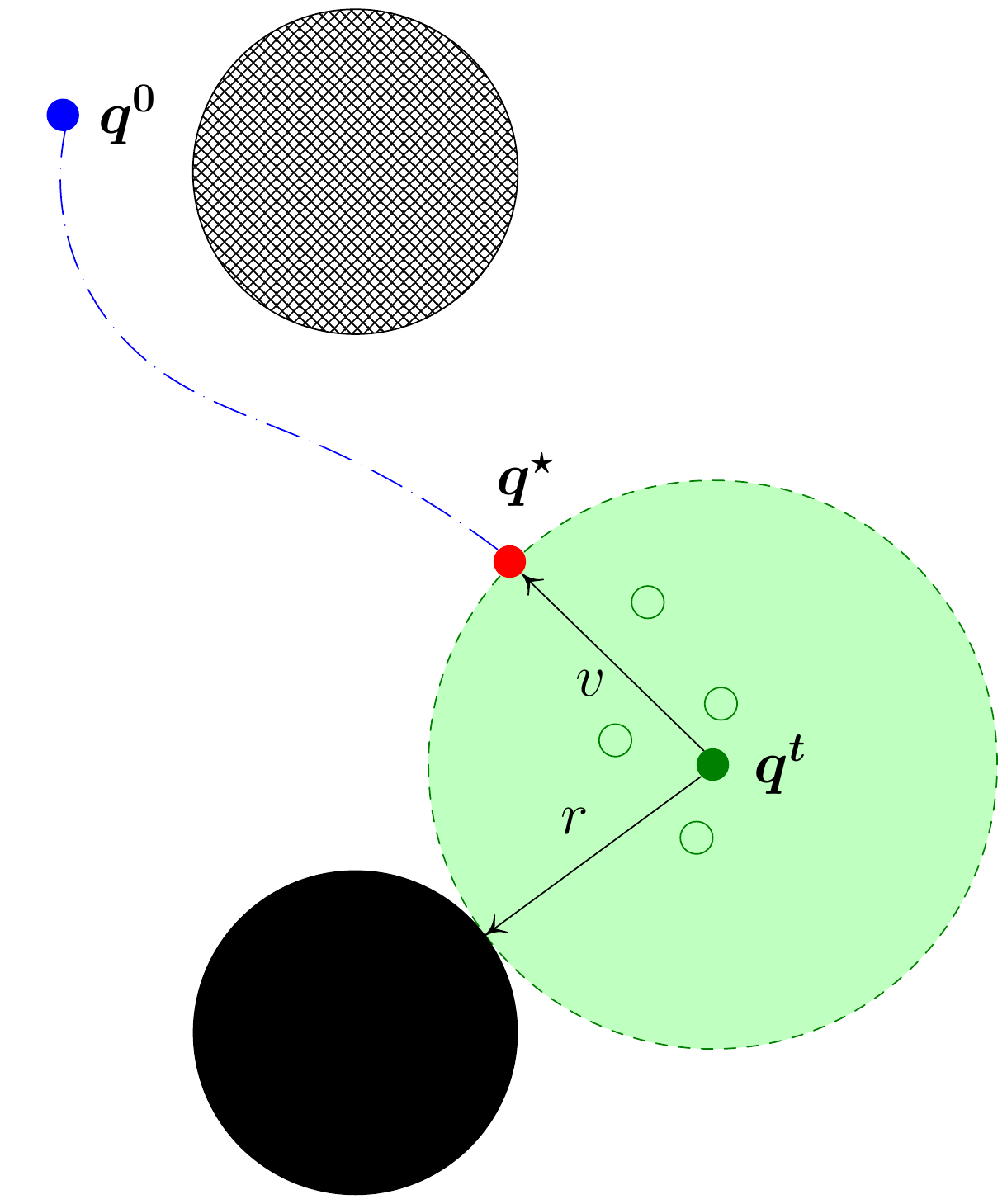}
		\caption[]%
		{{\small Coordination outside $W_{r}$}}    
		\label{fig:2}
	\end{subfigure}
	\begin{subfigure}[b]{0.2\textwidth}   
		\centering 
		\hspace*{-1.5cm}\includegraphics[scale=0.33]{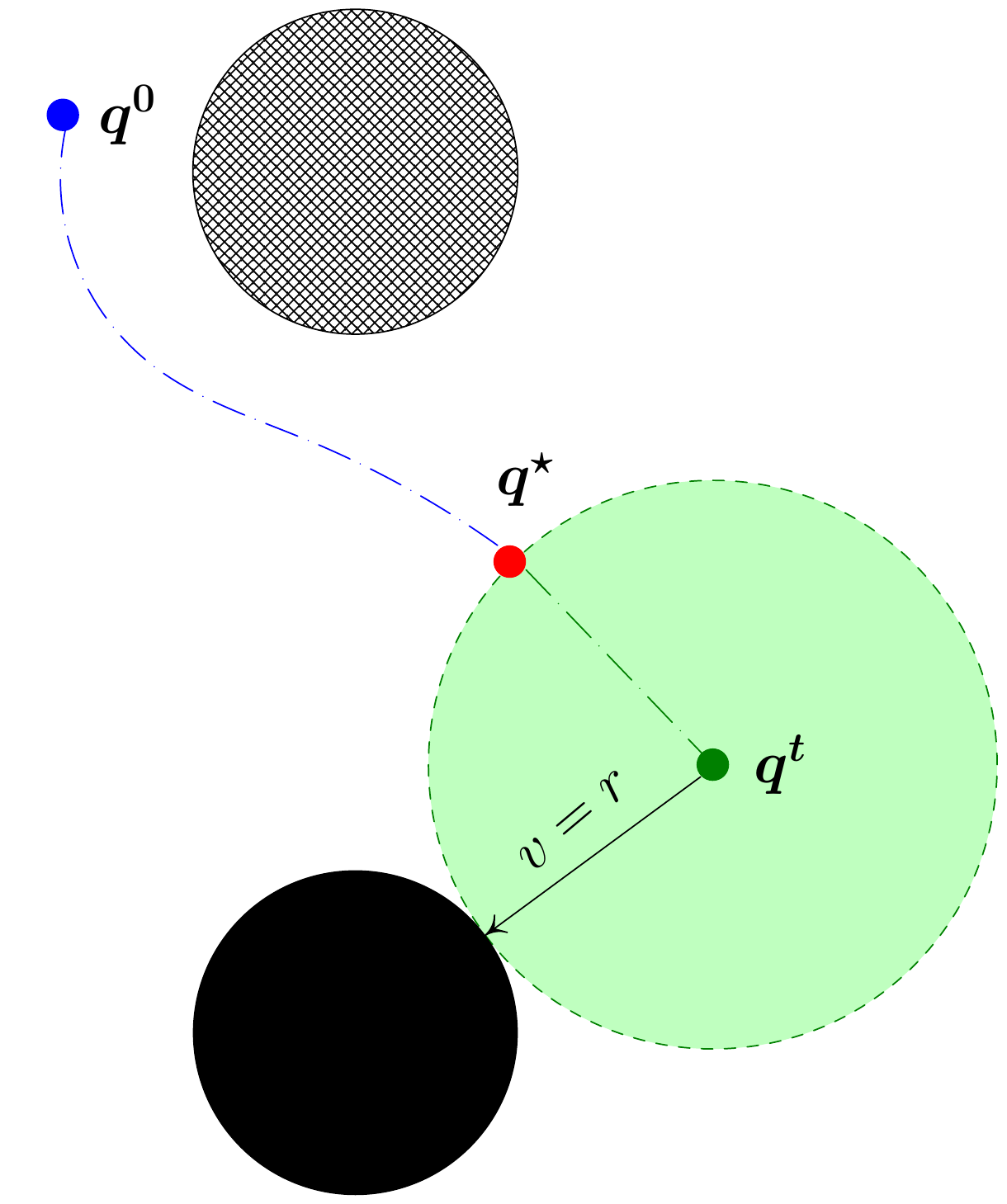}
		\caption[]%
		{{\small Coordination inside $W_{r}$}}    
		\label{fig:3}
	\end{subfigure}
	\begin{subfigure}[b]{0.2\textwidth}  
		\hspace*{-6mm}\raisebox{5mm}{\includegraphics[scale=0.52]{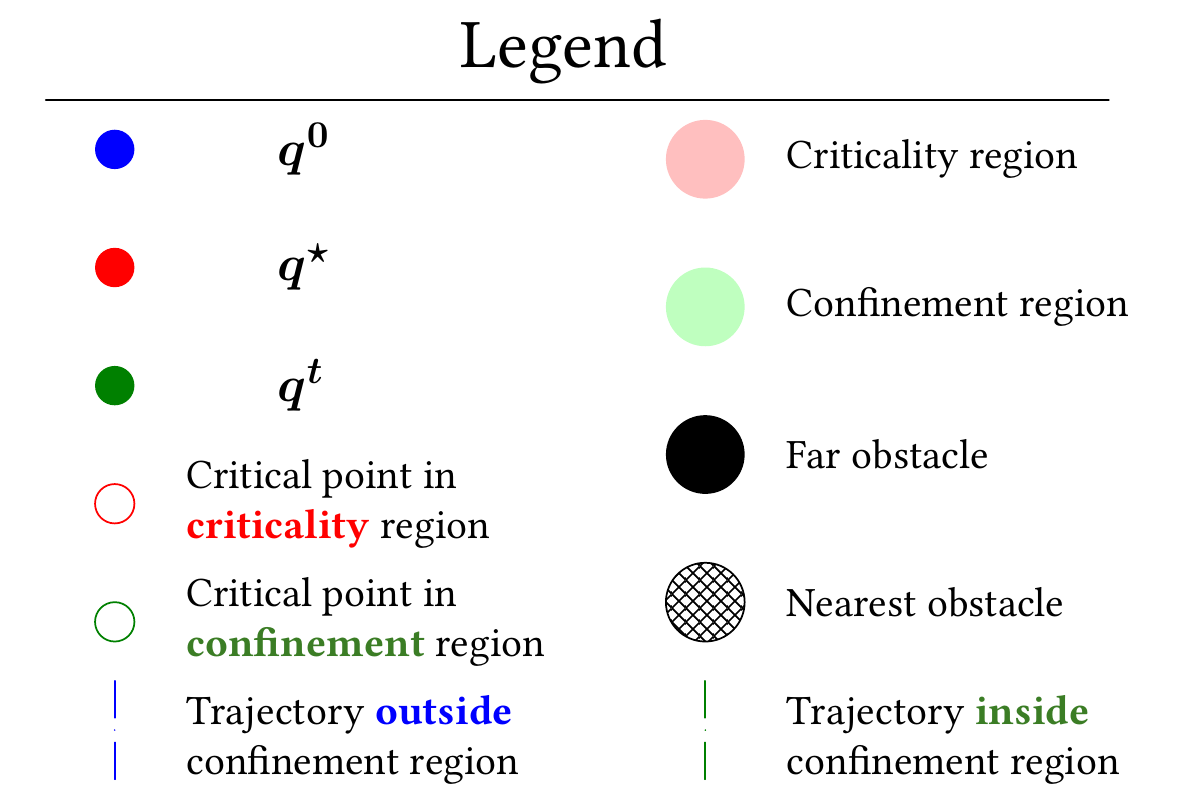}}
		%\caption[]%
		%{{\small Network 2}}    
		%\label{fig:3}
	\end{subfigure}
	\caption[]
	{\small The coordination of an agent using an MNF} 
	\label{fig:coord}
\end{figure*}
\section{Simulations}
\label{sec:sim}
\subsection{Total Associations}
We already stated that agents may be completely unselfish, that is, each agent totally contributes to the associative terms of the MNFs corresponding to all of its peers. As the first scenario, we take a coordination problem including 6 agents and 4 obstacles. The agents seek their target spots in the course of such total associations. Fig. \ref{fig:ttrj} illustrates the computed trajectories corresponding to each agent. One observes that the desired convergence are fully achieved. Moreover, the two-phase dynamics of MNF-based coordination are visible. In particular, the confined function of each agent first transfers it to its boundary CP associated with its optimal confinement factor. Then, a straight trajectory toward its target in the agent's confinement region is computed using its attraction kernel. The potential variations of the agents are depicted in Fig. \ref{fig:tpot} in which all potentials are completely vanished at target points.
\begin{algorithm}[tb]
	\caption{MNF Coordinator}
	\label{alg:MNFC} 
	Initial positions vector $\bm{\mathcal{Q}^{0}} := \{\bm{q_{i}^{0}}|i\in\mathcal{I} \}$\\
	Initial Target positions vector $\bm{\mathcal{Q}^{t}} := \{\bm{\bm{q_{i}^{t}}}|i\in\mathcal{I} \}$\\
	Obstacle positions vector $\bm{\mathcal{O}} := \{o_{j}|j\in\mathcal{J} \}$\\
	Weighting factors $\lambda_{1}, \lambda_{2}, \lambda_{3}, \beta$\\
	Confinement factor $\alpha$\\
	Motion step size factor $\gamma$\\
	MNF $\phi$\\
	\KwResult{Trajectories vector $\mathcal{Z}$ (including a sequence of points corresponding to each agent leading to its target point from its initial point)}
	\algrule[1pt]
	$\bm{\mathcal{Q}} \leftarrow \bm{\mathcal{Q}^{0}$}\\ $\mathcal{Z} \leftarrow \varnothing$\\pass $(\bm{\mathcal{Q}^{T}},\bm{\mathcal{O}}, \lambda_{1}, \lambda_{2}, \lambda_{3}, \alpha)$ to $\phi$\\
	\ForEach{agent $i$}{
		$\alpha^\dagger_{i} \leftarrow $ solve (\ref{eq:condition})\\
		\While{$q_{i} \not\in W_{r}^{i}$}{
			$\mathcal{Z}_{i} \leftarrow \text{GDC}(\psi_{i},q_{i}^{0},q_{i}^{\star},\gamma)$
		}
		\While{$q_{i} \in W_{r}^{i}$}{
			$\mathcal{Z}_{i} \leftarrow \mathcal{Z}_{i}~ \dot{\cup}~ \text{GDC}(\omega_{i},q_{i}^{0},q_{i}^{t},\gamma)$
		}
		$\mathcal{Z} \leftarrow \mathcal{Z} ~\dot{\cup}~ \mathcal{Z}_{i}$
	}
	\Return $\mathcal{Z}$
\end{algorithm}

To compare the performance of DNF and MNF in terms of convergence speed, we define a particular temporal measure. Namely, given a coordination scenario including some agents and two instances of MNF and DNF, let $\kappa_{\text{DNF}}^{i}$ and $\kappa_{\text{MNF}}^{i}$ be the required time for the agent $i$ to reach its target from its initial point using the DNF and the MNF, respectively. Thus convergence factor $\kappa$ reads
\begin{equation}
\kappa := \frac{\max\limits_{i} \kappa_{\text{MNF}}^{i}}{\max\limits_{i} \kappa_{\text{DNF}}^{i}}.
\end{equation}
This factor overall expresses the relative convergence speed with respect to the cited functions. Table \ref{tbl:TA} presents 4 simulated scenarios with various numbers of agents. The density of each scenario\footnote{The density of a multi-agent coordination scenario is defined as the overall number of the agents and obstacles divided by the area of the scenario's configuration space. This factor basically measures how much space agents have to be coordinated.} is increased gradually to study how MNFs react to more intensive collision-prone scenarios compared to the DNF. First, the convergence factor $\kappa$ is strictly less than one in all scenarios. So, one observes that coordination using MNFs are faster than those which are governed by DNFs. In particular, the more dense system we have, the faster MNFs are compared to DNF. Second, the more dense a coordination scenario is, the smaller the associative weighting factor $\lambda_{3}$ corresponding to agents has to be as a safety requirement. Namely, the associative stimulations applied to an agent may be hazardous in the situations in which the agent is often very close to obstacles. Thus, smaller associative factors are recommended in dense coordination problems. The optimal confinement factors corresponding to each scenario are reported, as well.
\begin{table}
	\centering
	\caption{Total Associations}
	\begin{tabular}
		[t]{>{\centering}m{1.3cm}>{\centering}m{1.3cm}>{\centering}m{0.8cm}>{\centering}m{0.8cm}>{\centering}m{1cm}>{\centering\arraybackslash}m{1cm}}
		\toprule
		\vspace*{-5mm}
		\multirowcell{5}{Configuration\\ Setups}&Agents&20&30&80&100\\
		&Obstacles&12&21&40&50\\
		\cmidrule(lr){3-4}
		\cmidrule(lr){5-6}
		&Area &\multicolumn{2}{c}{30$\times$15}&\multicolumn{2}{c}{40$\times$25}\\
		\midrule
		\multirowcell{4}{Potential\\Setups}&$\beta$ &\multicolumn{4}{c}{10}\\&$\lambda_1$&\multicolumn{4}{c}{0.4}\\
		&$\lambda_2$ &12&13&14&15\\
		&$\lambda_3$ &0.001&0.0005&0.0002&0.0001\\
		\midrule
		\multirowcell{2}{Results}&$\kappa$&0.734&0.681&0.593&0.502\\
		&$\alpha^\dagger$ &5.396&7.410&8.198&7.364\\
		\bottomrule
	\end{tabular}
	\label{tbl:TA}
\end{table}%
\subsection{Partial Associations}
We investigate the influence of partial associations on the coordination outcomes. We first analyze a partially associative version of the example presented in the previous section. In particular, we divide the agents into two coalitions, say, robots 1, 3, and 5, versus robots 2, 4, and 6. The computed trajectories and potential variations of this scenario are reflected in Fig. \ref{fig:ptrj} and \ref{fig:ppot}, respectively. The convergence times are trivially longer than those of the total associative case. This observation is justified by the fact that associative terms of MNFs in the partial case are smaller than those of the total case.

Given a fixed set of agents, obstacles, and swarm density, we consider various numbers of coalitions corresponding to this swarm. Namely, we uniformly partition agents into coalitions as expressed in Table \ref{tbl:PA}. One remarks that the convergence factor decreases as the number of coalitions increases because when less coalitions are taken into account, each agent is supported by more allies compared to the case in which the agent is affiliated with a smaller coalition.
\begin{table}[]
	\centering
	\caption{Partial Associations}
	\begin{tabular}[t]{>{\centering}m{1.3cm}>{\centering}m{1.3cm}>{\centering}m{0.8cm}>{\centering}m{0.8cm}>{\centering}m{1cm}>{\centering\arraybackslash}m{1cm}}
		\toprule
		\vspace*{-5mm}
		\multirowcell{7}{Configuration\\ Setups}&Agents&\multicolumn{4}{c}{100}\\
		&Obstacles&\multicolumn{4}{c}{50}\\
		\cmidrule(lr){3-6}
		&Area &\multicolumn{4}{c}{40$\times$25}\\
		\cmidrule(lr){3-6}
		&Coalitions&5&10&20&50\\
		\midrule
		\multirowcell{4}{Potential\\Setups}&$\beta$ &\multicolumn{4}{c}{10}\\&$\lambda_1$&\multicolumn{4}{c}{0.4}\\
		&$\lambda_2$ &12&13&14&15\\
		&$\lambda_3$ &0.001&0.0005&0.0002&0.0001\\
		\midrule
		\multirowcell{1}{Results}&$\kappa$&0.522&0.47&0.602&0.689\\
		\bottomrule
	\end{tabular}
	\label{tbl:PA}
\end{table}%

The applied simulations also reveal the impact of agents' initial conditions on how partial associative scenarios may tend to behave like either DNF-based (selfish and isolated agents) or total MNF-based swarms. In particular, when agents are relatively close to their targets, their contributions to the associative terms of their peers' MNFs are not noticeable. In other words, a partially associative MNF resembles the behavior of DNFs if its corresponding multi-agent system is already fairly coordinated. In a social perspective, an agent which is far from its target may be isolated from its peers, if they are extremely closer to their targets compared to its distance to its own target. In contrast, if the majority of agents are relatively far from their destinations, then the agents will be stimulated by each other's peers for a longer time. In this case, the dynamics of the partially associative MNFs more resembles a totally associative one rather than a DNF.
\begin{figure}
	\centering\includegraphics[scale=0.5]{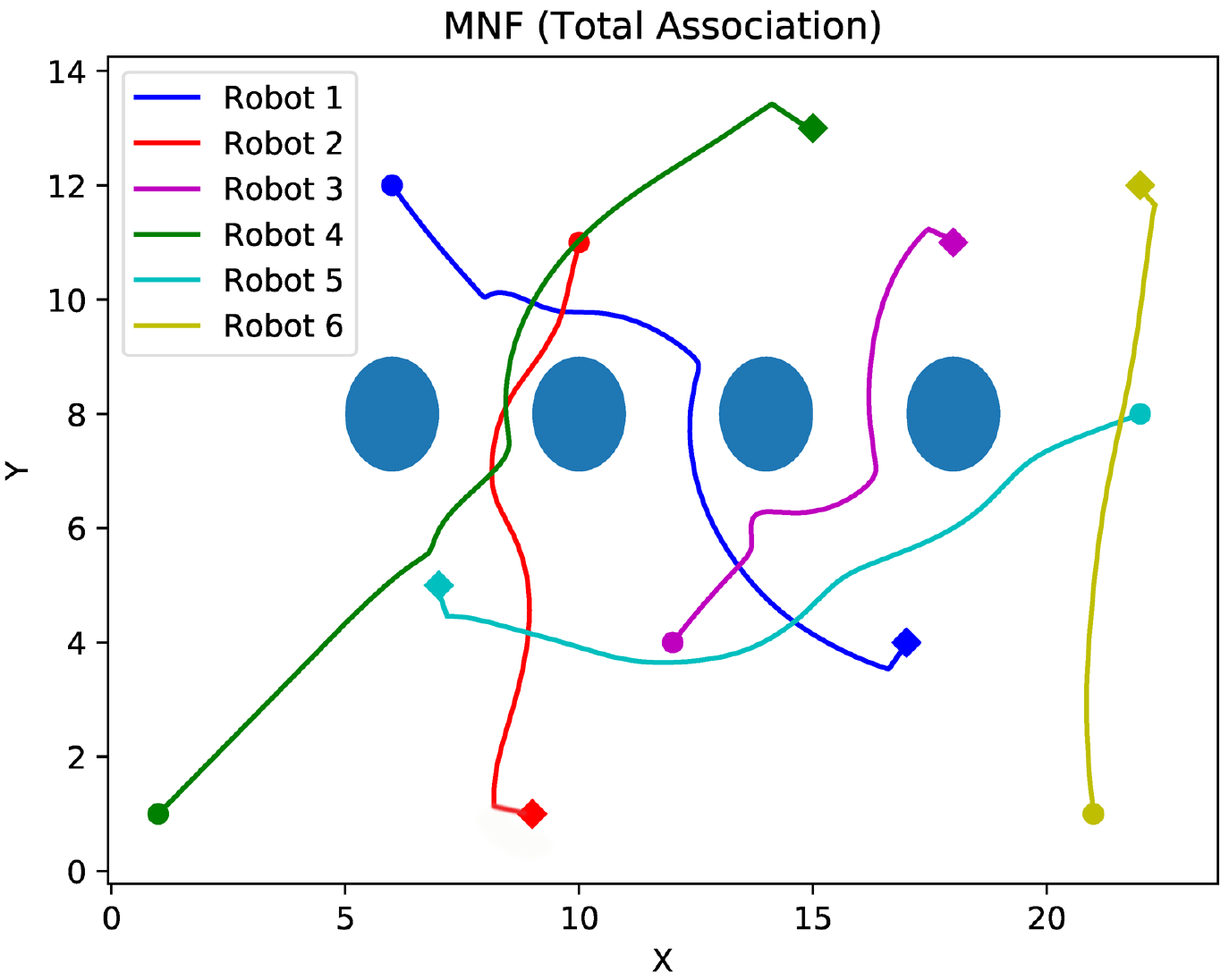}
	\caption{A typical trajectory evolution using MNF under total associations}
	\label{fig:ttrj}
\end{figure}
\begin{figure}
	\centering
	\includegraphics[scale=0.5]{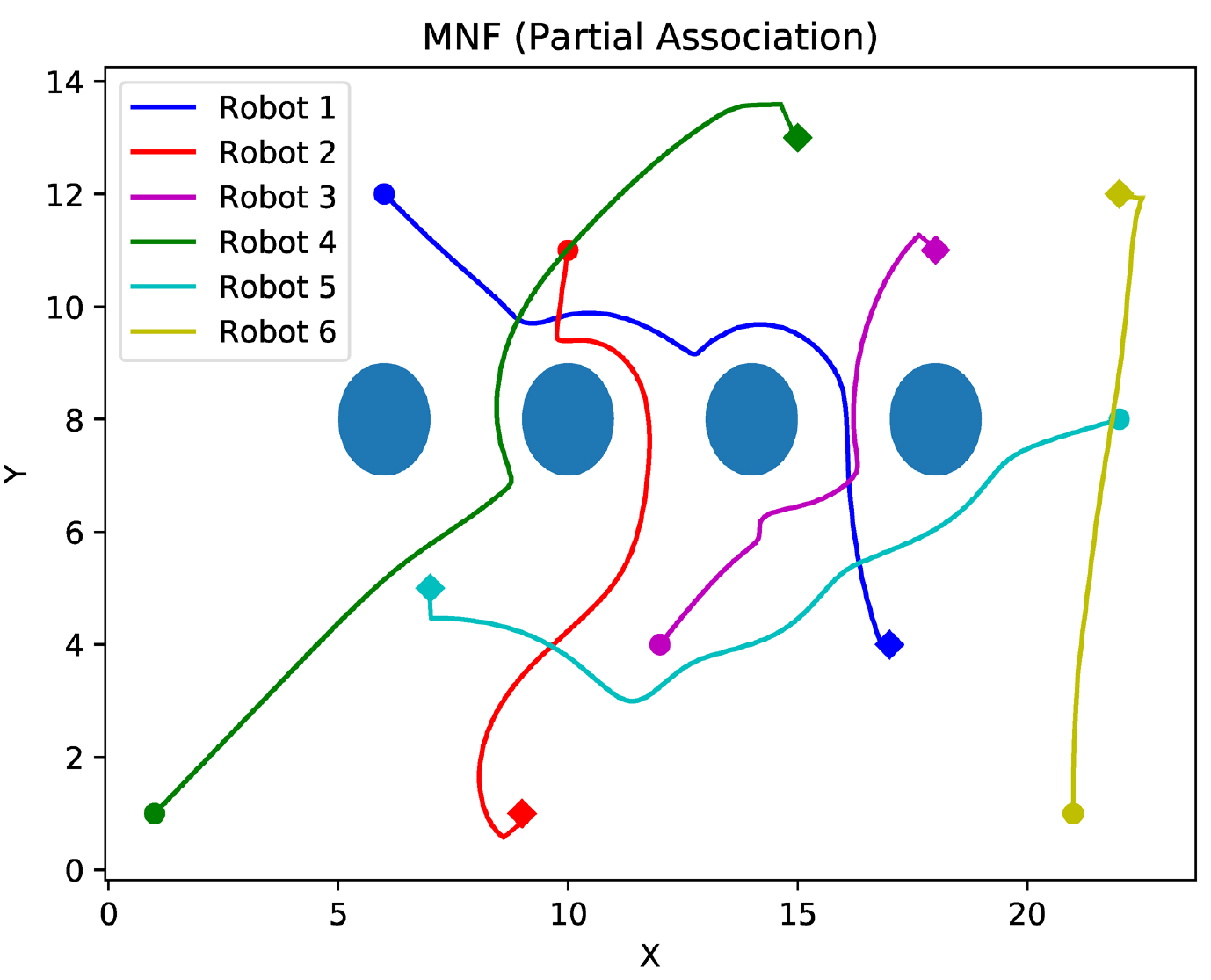}
	\caption{A typical trajectory evolution using MNF under partial associations.}    
	\label{fig:ptrj}
\end{figure}
\begin{figure}[t]
	\centering\includegraphics[scale=0.32]{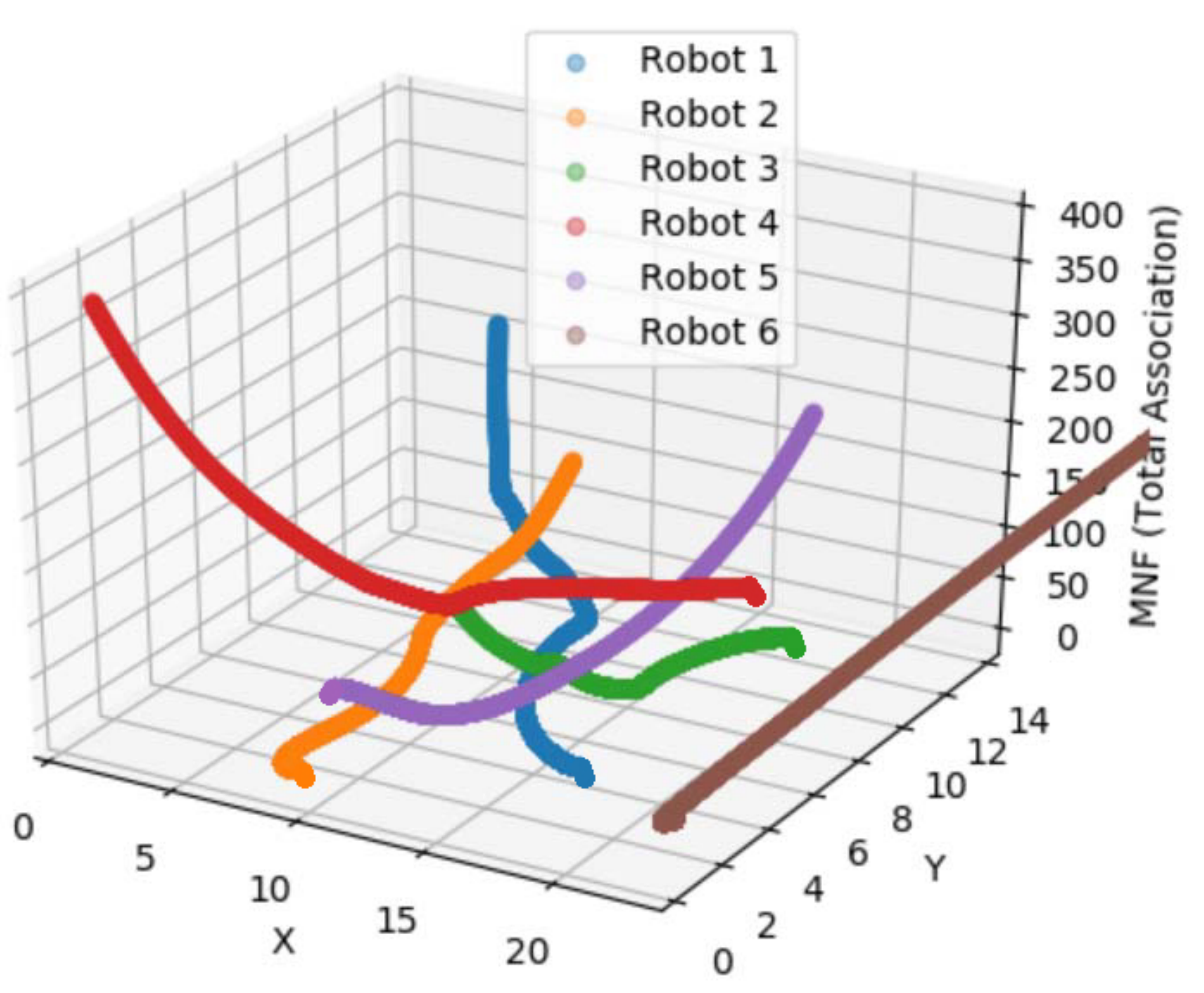}
	\caption{A typical potential evolution using MNF under total associations} 
	\label{fig:tpot}
\end{figure}
\begin{figure}[t]  
	\centering
	\includegraphics[scale=0.32]{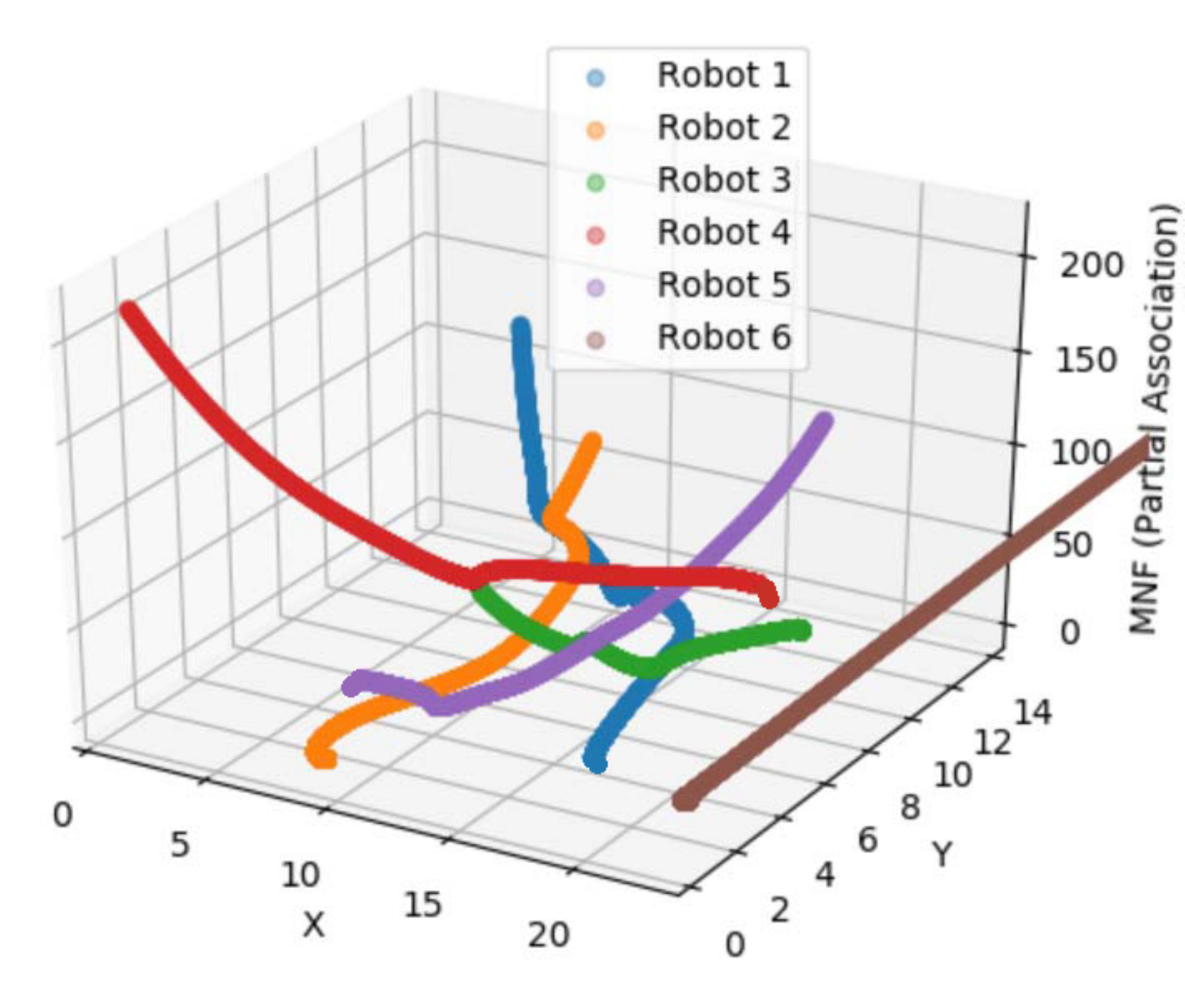}
	\caption{A typical potential evolution using MNF under partial associations.}   
	\label{fig:ppot}
\end{figure}
\section{Conclusions}
\label{sec:conc}
This paper introduces a new class of artificial potential fields whose instances, contrary to decentralized navigation function (DNF), do not need to be Morse. Instead, the coordination process of an agent is done by a meta navigation function (MNF) in two phases. First, a confined function coordinates the agent to its farthest CP with respect to its target point. During this phase, the agent cooperate to its peers to send and receive associative attractions. In the second phase, an attraction kernel drives the agent to its target in its confinement region which includes no other obstacles. The adaptive formulation of each MNF instance paves the way for realizing more social behaviors for agents during their coordination, which is out of the scope of the traditional DNFs' abilities. MNFs also surpass DNFs in faster convergence. The applied Morse relaxation makes the family of MNFs wider than that of DNFs. So, it is easier to design effective MNFs compared to DNFs. The association model presented in this paper is total, in that each agent equally contributes to the dynamics of all of the other agents of its system. The results demonstrate that such a united behavior leads to fast coordination of multi-agent systems. In a similar vein, one may take partial associations into account. In this case, each agent may belong to a particular coalition constructed by a subset of the agents of a system. So, the agents of each coalition cooperate with each other, but they are indifferent to other coalitions except being influenced by their hostile repulsive impacts.

\bibliographystyle{IEEEtr}
\bibliography{references}{}	

\begin{thebibliography}{10}

\bibitem{siciliano2016springer}
B.~Siciliano and O.~Khatib, {\em Springer handbook of robotics}.
\newblock Springer, 2016.

\bibitem{warren1989global}
C.~W. Warren, ``Global path planning using artificial potential fields,'' in
  {\em Proceedings, 1989 International Conference on Robotics and Automation},
  pp.~316--321, Ieee, 1989.

\bibitem{fahimi2009autonomous}
F.~Fahimi, ``Autonomous robots,'' {\em Modeling, Path Planning and Control/F.
  Fahimi--New York: Springer}, 2009.

\bibitem{macktoobian2016time}
M.~Macktoobian and M.~A. Shoorehdeli, ``Time-variant artificial potential field
  (tapf): a breakthrough in power-optimized motion planning of autonomous space
  mobile robots,'' {\em Robotica}, vol.~34, no.~5, pp.~1128--1150, 2016.

\bibitem{tanner2012multiagent}
H.~G. Tanner and A.~Boddu, ``Multiagent navigation functions revisited,'' {\em
  IEEE Transactions on Robotics}, vol.~28, no.~6, pp.~1346--1359, 2012.

\bibitem{palais2000generalized}
R.~Palais and S.~Smale, ``A generalized morse theory,'' in {\em The Collected
  Papers of Stephen Smale: Volume 2}, pp.~503--510, 2000.

\bibitem{de2006formation}
M.~C. De~Gennaro and A.~Jadbabaie, ``Formation control for a cooperative
  multi-agent system using decentralized navigation functions,'' in {\em 2006
  American Control Conference}, pp.~6--pp, IEEE, 2006.

\bibitem{macktoobian2021astrobotics}
M.~Macktoobian, D.~Gillet, and J.-P. Kneib, ``Astrobotics: swarm robotics for
  astrophysical studies,'' {\em IEEE Robotics \& Automation Magazine}, 2021.

\bibitem{macktoobian2021data}
M.~Macktoobian, F.~Basciani, D.~Gillet, and J.-P. Kneib, ``Data-driven
  convergence prediction of astrobots swarms,'' {\em IEEE Transactions on
  Automation Science and Engineering}, 2021.

\bibitem{macktoobian2019complete}
M.~Macktoobian, D.~Gillet, and J.-P. Kneib, ``Complete coordination of robotic
  fiber positioners for massive spectroscopic surveys,'' {\em Journal of
  Astronomical Telescopes, Instruments, and Systems}, vol.~5, no.~4, p.~045002,
  2019.

\end{thebibliography}
\end{document}